\documentclass[11pt]{article}
\usepackage{fullpage}
\usepackage{graphicx,amssymb,amsmath}
\usepackage{amsthm}

\usepackage{textcomp}
\usepackage{xr-hyper}

\usepackage[utf8]{inputenc}
\usepackage{amsfonts}
\usepackage{mathtools}
\usepackage{physics}
\usepackage{comment,float}
\usepackage{pgfplots}
\usepgfplotslibrary{groupplots}

\usepackage[ruled,vlined]{algorithm2e}

\usepackage{hyperref}
\hypersetup{
    colorlinks=true,
    linkcolor=blue,
    filecolor=magenta,      
    urlcolor=cyan,
}

\usepackage{subfig}
\newtheorem{theorem}{Theorem}
\newtheorem{lemma}{Lemma}

\newcommand{\OPT}{\texttt{OPT}}

\newcommand{\mylabel}{\mathsf{label}}
\newcommand{\NN}{\mathsf{NN}}
\newcommand{\Lip}{\mathsf{Lip}}

\newcommand{\diam}{\mathsf{diam}}

\newcommand{\ball}{\mathsf{ball}}
\newcommand{\corruption}{\mathsf{corruption}}
\newcommand{\eps}{\varepsilon}

\newcommand{\AlgKNNFlip}{\textsf{Poison-$k$-NN}}

\newcommand{\gls}[1]{\textbf{#1}}

\title{Geometric Algorithms for  $k$-NN Poisoning}

\author{
    Diego Ihara Centurion\thanks{Department of Computer Science, University of Illinois at Chicago, \tt{\{dihara2, kchuba2, bfan4, fsgher2, tsivap2, sidiropo, astrai3\}@uic.edu}}
    \and
    Karine Chubarian\footnotemark[\value{footnote}]
    \and
    Bohan Fan\footnotemark[\value{footnote}]
    \and
    Francesco Sgherzi\footnotemark[\value{footnote}]
    \and
    Thiruvenkadam S Radhakrishnan\footnotemark[\value{footnote}]
    \and
    Anastasios Sidiropoulos\footnotemark[\value{footnote}]
    \and
    Angelo Straight\footnotemark[\value{footnote}]
}

\date{June 2023}

\index{Ihara, Diego}
\index{Chubarian, Karine}

\pgfplotsset{compat=1.18} 

\begin{document}

\maketitle

\begin{abstract}
We propose a label poisoning attack on geometric data sets against  $k$-nearest neighbor classification. We provide an algorithm that can compute an $\eps n$-additive approximation of the optimal poisoning in $n\cdot 2^{2^{O(d+k/\eps)}}$ time for a given data set $X \in \mathbb{R}^d$, where $|X| = n$. Our algorithm achieves its objectives through the application of multi-scale random partitions.  
\end{abstract}

\section{Introduction}
Recent developments in machine learning have spiked the interest in robustness, leading to several results in \textit{adversarial machine learning}~\cite{BaiL0WW21, Biggio2018WildPT, huang2011adversarial}.
A central goal in this area is the design of algorithms that are able to impair the performance of traditional learning methods by adversarially perturbing the input ~\cite{Carlini21, PitropakisPGAL19, SadeghiBG20}. 
Adversarial attacks can be exploratory, such as evasion attacks,
or causative, poisoning the training data to affect the performance of a machine learning algorithm or attack the algorithm itself. 
Backdoor poisoning is a type of causative adversarial attack, in which the attacker has access to the whole or a portion of the training data that they can perturb. 
Clean-label poisoning attacks are a type of backdoor poisoning attack that perturb only the features of the training data leaving the labels untouched, so as to make the poison less detectable. In the other end of the spectrum are label poisoning attacks that perturb or flip the training data labels. 

\paragraph{Why compute provably nearly-optimal poison attacks?}
A limitation with current poisoning methods is that it is not possible to adversarially perturb an input so that the performance of \emph{any} algorithm is negatively affected.
Moreover, it is generally not clear how to provably compare different poisoning methods.
We seek to address these limitations of adversarial machine learning research using tools from computational geometry.

Specifically, we study the following optimization problem: Given some data set, $X$, compute a small perturbation of $X$, so that the performance of a specific classifier deteriorates as much as possible.
An efficient solution to this optimal poisoning problem can be used to compare the performance of different classification algorithms, as follows.
Suppose we want to compare the performance of a collection of classification algorithms, $A_1, ..., A_t$, on some fixed data set $X$, in the presence of a poisoning attack that produces a bounded perturbation, $X'$, of $X$.
Ideally, we would like to have provable worst-case guarantees on the robustness of $A_1,...,A_t$.
However, such results are often hard to prove rigorously, and thus many existing methods lack such guarantees.
Since the poisoned data set $X'$ is unknown, we cannot simply run $A_1$, ..., $A_t$ on $X'$ and compare the results.
Instead, our method allows us to compute from $X$ some poisoned data set, $X''$, which is provably a nearly-optimal poison against the specific classification task.

\subsection{Robustness of $k$-Nearest Neighbors}

We instantiate the above general optimization problem of computing nearly-optimal poison attacks to the specific task of $k$-nearest neighbor classification.
Nearest-neighbor based algorithms are naturally robust due to the presence of an inherent majority voting mechanism. 
In \cite{knnrobust2022}, they are used to provide individual and joint certifications for test predictions in the presence of data poisoning and backdoor attacks. 
In \cite{paudice2018label}, a defense algorithm is proposed using $k$-nearest neighbors against label-flipping attacks. 
However, computing provably nearly-optimal poisoning against such algorithms has not been studied prior to our work. 
We provide an approximation algorithm that computes an optimal label flipping poisoning attack against $k$-nearest neighbors that achieve provable guarantees.

\subsection{Our Results}
We design and analyze poisoning algorithms against $k$-nearest neighbor classification ($k$-NN) in the setting of binary label classification.
The $k$-NN classifier is arguably one of the most popular and well-studied methods used in machine learning and geometric data analysis
\cite{fix1989discriminatory}.
The classifier works as follows:
Given a set of labeled points, $X\subset \mathbb{R}^d$, and some unlabeled $p\in \mathbb{R}^d$, we can compute a label for $p$ by taking the most frequently occurring label in the multiset of labels of the $k$ nearest neighbors of $p$ in $X$.

We formulate the poisoning problem against $k$-NN as follows.
Given some set of points, $X$, with binary labels, and some $m\in \mathbb{N}$, the goal is to flip the labels of at most $m$ points, so that we maximize the number of points in $X$
for which their \emph{predicted} label differs from their true label.
We refer to the set of points with flipped labels as an $m$-poison and define the number of points for which their predicted label differs from the original label as the \emph{corruption}.
The following summarizes our results.

\begin{theorem}\label{thm:kNN}
On input $X\subset \mathbb{R}^d$, with $|X|=n$, and $m\in \mathbb{N}$,
Algorithm \AlgKNNFlip~computes a $m$-poison against $k$-NN,
with expected corruption $\OPT_m(X)-\eps n$, in time $n\cdot 2^{2^{O(d+k/\eps)}}$,
where $\OPT_m(X)$ denotes the maximum corruption of any $m$-poison.
\end{theorem}

While the above problem formulation only involves a fully labeled set $X$, typical tasks involve a training set $X_{train}$ and a  test set $X_{test}$. 
In order to address this case, we modify the algorithm in Theorem \ref{thm:kNN} so that it computes a poison of the training set, so that the prediction error on the test set deteriorates as much as possible.
This result is summarized in the following.

\begin{theorem}\label{thm:kNN2}
On input $X_{train}, X_{test}\subset \mathbb{R}^d$, with $|X_{train}|=n_{train}$,
$|X_{test}|=n_{test}$,
and $m\in \mathbb{N}$,
Algorithm \AlgKNNFlip'~computes a $m$-poison against $k$-NN,
with expected corruption $\OPT_m(X_{train}, X_{test})-\eps n_{test}$, in time $(n_{train}+n_{test})\cdot 2^{2^{O(d+k/\eps)}}$,
where $\OPT_m(X_{train}, X_{test})$ denotes the maximum corruption 
incurred on $X_{test}$
when all neighbors are chosen from $X_{train}$,
of any $m$-poison on $X_{train}$.
\end{theorem}

Algorithm \AlgKNNFlip'~is an adaptation of \AlgKNNFlip, and has a similar analysis.
Algorithm \AlgKNNFlip~uses as a subroutine a procedure for computing a random partition of a metric space.
The random partition has the property that the diameter of each cluster is upper bounded by some given Lipschitz function, while the probability of two points being separated is upper bounded by a multiple of their distance divided by the cluster diameter (see Section \ref{sec:random_partitions} for a formal definition).
This is inspired by the multi-scale random partitions invented by Lee and Naor \cite{lee2005extending} in the context of the Lipschitz extension problem.
We believe that our formulation could be of independent interest.

\subsection{Related Work}
To the best of our knowledge, no prior work tackles the adversarial poisoning of geometric algorithms giving provable  bounds.
The most traditional poisoning method is the \gls{fsgm} \cite{goodfellow2015}, which consists in adding, to each testing sample, noise with the same dimensionality of the input and proportional to the gradient of the cost function in that point.
This approach is proven to be the most damaging adversarial example against linear models like logistic regression.
However, it is less effective on deep neural networks, as they are able to approximate arbitrary functions \cite{HORNIK1989359}.
\gls{pgd} \cite{madry2019deep} improves upon \gls{fsgm} by iteratively looking for better adversarial examples for a given input toward the direction of the increase of the cost function. However, although producing \textit{better} adversarial samples, it still encounters the same drawbacks of \gls{fsgm}.

There are a few existing algorithms that perform label flipping attacks. In \cite{paudice2018label} they use a greedy algorithm to flip the examples that maximize the error on a validation set, when the classifier is trained on the poisoned dataset, and use k-NN to reassign the label in the training set as the defense against this type of label flipping attacks. \cite{labelcont2017} model the label attacks as a bilevel optimization problem targeting linear classifiers and also experiment with the transferability of these attacks.
Traditional poisoning methods, however, do not offer provable guarantees on the reduction in performance, thus yielding results that are not \textit{problem} dependent but rather \textit{implementation} or \textit{model} dependent~\cite{ChhabraRM20, Chhabra21, CINA2022108306}.

There have also been a few defenses proposed against label flipping attacks. In \cite{rosenfeld2020certified} they propose a pointwise certified defense against adversarially manipulated data up to some “radius of perturbation” through randomized smoothing. Specifically, each prediction is certified robust against a certain number of training label flips.

\subsection{Notation}
For any $k\in \mathbb{N}$, let $[k]=\{1,\ldots,k\}$.
Let $M=(X,\rho)$ be a metric space.
For any $X'\subseteq X$, we denote by $\diam_M(X')$ the diameter of $X'$, i.e.~$\diam_M(X')=\sup_{x,y\in X'} \rho(x,y)$; we also write $\diam(X')$ when $M$ is clear from the context.
For any $x\in X$, $Y\subseteq X$, we write $\delta(x,Y)=\inf_{y\in Y} \rho(x,y)$.
For any metric space $M=(X,\rho)$, any finite $Y\subset X$, any $i\in \mathbb{N}$, and any $q\in X$, let $\NN_i(q,Y)$ denote the $i$-th nearest neighbor of $q$ in $Y$.

\subsection{Organization}
\label{sec:notation}

The rest of the paper is organized as follows.
Section \ref{sec:random_partitions} presents our result on random partitions of metric spaces.
Section \ref{sec:kNN} presents our algorithm for poisoning against $k$-NN classifiers.
We conclude and highlight some open problems in Section \ref{sec:conclusions}.

\section{Random partitions of metric spaces}
\label{sec:random_partitions}

\begin{figure*}[ht]
    \centering
    \input{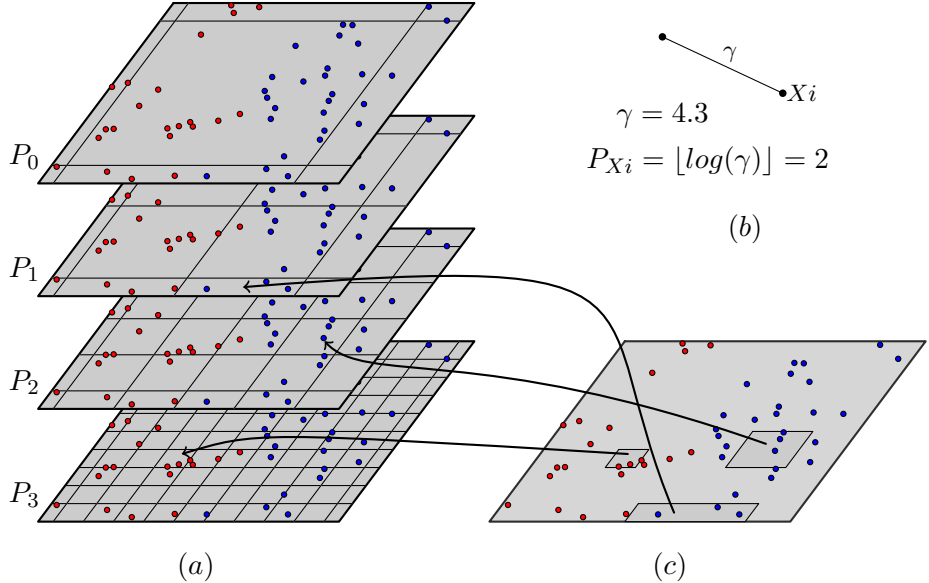}
    \caption{Illustration of the application of the multi-scale random partition approach to a set of points. (a) We begin with a random partition and refine to produce an additional level. (b) The selection of the level depends on the distance to the $k$-th neighbor. (c) The resulting partition is the union of cells originating at different levels of granularity.}
    \label{fig:multiscale_part}
\end{figure*}

In this Section, we introduce a random metric space partitioning scheme. The main idea is to partition a given metric space so that the radius of each cluster is bounded by some Lipschitz function, while ensuring that only a small fraction of pairs end up in different clusters, in expectation.
This partition is used by our algorithm for partitioning the problem into several sub-problems for each cluster.

For any partition $P$ of a ground set $Y$, and for any $y\in Y$, we denote by $P(y)$ the unique cluster in $P$ that contains $y$.
Let $M=(X,\rho)$ be some metric space.
Let $P$ be a random partition of $M$.
For any $\Delta>0$, 
we say that $P$ is $\Delta$-bounded if, with probability 1, for all clusters $C\in P$, we have $\diam(C)\leq \Delta$.
For any $\beta>0$, we say that $P$ is $\beta$-Lipschitz if for all
$x,y\in X$,
\[
\Pr[P(x)\neq P(y)] \leq \beta \frac{\rho(x,y)}{\Delta}.
\]
The infimum $\beta>0$ such that for all $\Delta>0$, $M$ admits a $\Delta$-bounded $\beta$-Lipschitz random partition, is referred to as the \emph{modulus of decomposability} of $M$, and is denoted by $\beta_M$.

\begin{lemma}[\cite{charikar1998approximating}]
\label{lem:mod_Rd}
For any $d\in \mathbb{N}$, and for subset of $d$-dimensional Euclidean space, $M$, we have that $\beta_M = O(\sqrt{d})$.
\end{lemma}

The following uses ideas from \cite {lee2005extending} and \cite{lee2009geometry}.

\begin{lemma}[Multi-scale random partition]
\label{lem:Lip_multiscale}
Let $C>0$.
Let $M=(X,\rho)$ be a metric space, and let $r:X \to \mathbb{R}_{\geq 0}$.
Then there exists a random partition $P$ of $M$, satisfying the following conditions:
\begin{description}
\item{(1)}
The following statement holds with probability 1:
For any $p\in X$,
\[
\diam(P(p)) \leq r(p) C^2 %
\]

\item{(2)}
For any $p,q \in X$, 
\[
\Pr_P[P(p)\neq P(q)] \leq \left(\frac{2 \|r\|_{\Lip}}{\log C} + \beta_M\right) \frac{\rho(p,q)}{r(p)} %
\]
\end{description}
Moreover, given $M$ as input, the random partition, $P$ can be sampled in time polynomial in $|X|$.
\end{lemma}

\begin{proof}
Let $B=\|r\|_{\Lip}$

For any $i\in \mathbb{Z}$, let $P_i$ be a $C^i$-bounded $\beta_M$-Lipschitz random partition of $M$.
Thus, each point $x\in X$ is mapped to some cluster in each $P_i$.
We construct $P$ by assigning each $x\in X$ to a single one of these clusters.
We first sample $\alpha\in [0,1]$, uniformly at random.
Then, for each $x\in X$, we assign $x$ to the cluster $P_{\lceil \alpha + \log_{C} (r(x)) \rceil}(x)$.
This concludes the construction of $P$.
It remains to show that the assertion is satisfied.

For any $p\in X$, we have that $P(p) \subseteq P_i(p)$, where $i=\lceil \alpha + \log_{C} (r(x)) \rceil \leq 2+\log_{C}(r(x))$.
Since $P_i$ is $C^i$-bounded, it follows that 
\begin{align*}
\diam(P(p)) &\leq \diam(P_i(p)) & \text{($P(p)\subseteq P_i(p)$)} \\
 &\leq C^i & \text{($P_i$ is $C^i$-bounded)} \\
 &\leq C^{2+\log_{C}(r(p))} \\
 &= r(p) C^{2},
\end{align*}
with probability 1,
which establishes part (1) of the assertion.

It remains to establish part (2).
Let $p,q\in X$.
We may assume, without loss of generality, that $0<r(p)\leq r(q)$.
Let ${\cal E}_1$ be the event that $\lceil \alpha + \log_{C}(r(p))\rceil \neq \lceil \alpha + \log_{C}(r(q)) \rceil$.
We have
\begin{align}
\Pr[{\cal E}_1] &= \Pr[\lceil \alpha + \log_{C}(r(p)) \rceil \neq \lceil \alpha + \log_{C}(r(q)) \rceil] \notag \\
 &\leq |\log_{C}(r(p)) - \log_{C}(r(q))| \notag \\
 &= \log_{C}\frac{r(q)}{r(p)} \notag \\
 &= \left(\log\frac{r(q)}{r(p)} \right) / \left(\log C\right) \notag \\
 &\leq (1/\log C) \log\frac{r(p)+B\cdot \rho(p,q)}{r(p)} \qquad\;\; \text{($\|r\|_{\Lip}=B$)} \notag \\
 &= (1/\log C) \log\left(1+ \frac{B\cdot \rho(p,q)}{r(p)}\right) \notag \\
 &\leq (1/\log C) 2B\frac{\rho(p,q)}{r(p)}. \label{eq:multi_E1}
\end{align}
Conditioned on $\neg{\cal E}_1$, there exists $t\in \mathbb{Z}$, such that 
$t=\lceil \alpha+\log_{C}(r(p))\rceil = \lceil \alpha+\log_{C}(r(q))\rceil$;
let ${\cal E}_2$ be the event that $P_t(p)\neq P_t(q)$.
Since $P_t$ is $C^t$-bounded $\beta_M$-Lipschitz, it follows that
\begin{align}
\Pr[{\cal E}_2] \leq \beta_M \frac{\rho(p,q)}{C^t}
\leq
\beta_M \frac{\rho(p,q)}{C^{\log_{C}(r(p))}}
= \beta_M \frac{\rho(p,q)}{r(p)}. \label{eq:multi_E2}
\end{align}
By \eqref{eq:multi_E1} and \eqref{eq:multi_E2} we obtain that
\[
\Pr[P(p)\neq P(q)] \leq ((1/\log C)2B + \beta_M) \frac{\rho(p,q)}{r(p)}, 
\]
which establishes part (2) of the assertion, and concludes the proof.
\end{proof}

Figure \ref{fig:multiscale_part} illustrates the partitioning process.

\section{Poisoning $k$-NN}
\label{sec:kNN}
In this Section, we describe our poisoning algorithm, which is our main result.

Let $d\in \mathbb{N}$, and let $X\subset \mathbb{R}^d$. Let $\mylabel:X\to \{1,2\}$, and let $k\in \mathbb{N}$ be odd, with $k\leq n$. For any $p\in \mathbb{R}^d$, for any $i\in [k]$, let $\Gamma_i(p)$ be an $i$-th nearest neighbor of $p$, breaking ties arbitrarily, and let
\[
\gamma_i(p) = \|p-\Gamma_i(p)\|_2.
\]
We write $\gamma(p)=\gamma_k(p)$.

\begin{lemma}\label{lem:1-Lip}
The function $\gamma:\mathbb{R}^d\to \mathbb{R}$ is 1-Lipschitz.
\end{lemma}

\begin{proof}
WLOG, let $\gamma(p)\ge \gamma(q)$. It is sufficient to prove that $\gamma(p)\le\norm{p-q}+\gamma(q)$, which means $\Gamma_k(p)$ is within distance $\norm{p-q}+\gamma(q)$ from $p$. 
Now consider two cases: 
\begin{description}
    \item{Case 1:} $\Gamma_k(p)$ falls in $\ball(q, \gamma(q))$. By triangle inequality, $\gamma(p)\leq \norm{p-q}+\norm{\Gamma_k(p)-q}\le\norm{p-q}+\gamma(q)$.
    
    \item{Case 2:} $\Gamma_k(p)$ falls outside of $\ball(q, \gamma(q))$.
    If $\gamma(p)>\norm{p-q}+\gamma(q)$, then all the $k$-nearest neighbour of $q$ are closer to $p$ than $\Gamma_k(p)$, which is a contradiction. 
\end{description}
\end{proof}

\begin{lemma}[Euclidean multi-scale random partition]
\label{lem:multiscale-euclidean}
Let $\eps>0$, there exists a random partition $P$ of $X$, satisfying the following conditions:
\begin{description}
    \item{(1)}
    The following statement holds with probability 1:
    For any $p\in X$, 
    \[
    \diam(P(p)) \leq \gamma(p)   2^{8k/\eps} O(\sqrt{d})
    \]
    
    \item{(2)}
    For any $p,q\in X$,
    \[
    \Pr[P(p) \neq P(q)] \leq  \frac{\eps \|p-q\|_2}{k \gamma(p)}.
    \]
\end{description}
Moreover, $P$ can be sampled in time polynomial in $|X|$.
\end{lemma}

\begin{proof}
Let $M=(X,\rho)$ be the metric space obtained by setting $\rho$ to be the Euclidean metric.
By Lemma \ref{lem:mod_Rd} we have
$\beta_M=O(\sqrt{d})$.
Let $P$ be the random partition of $X$ obtained by applying Lemma 
\ref{lem:Lip_multiscale},
setting $\gamma:X\to\mathbb{R}_{\geq 0}$ where 
$r = B \gamma$,
with 
$B = 2k \beta_M / \eps$,
and
$C = 2^{4k/\eps}$.
By Lemma \ref{lem:1-Lip} we have that $\|\gamma\|_{\Lip}=1$, and thus
$\|r\|_{\Lip}=\|B\gamma\|_{\Lip} = B\|\gamma\|_{\Lip}=B$.
The assertion now follows by straightforward substitution on Lemma \ref{lem:Lip_multiscale}.
\end{proof}

\begin{lemma}
\label{lem:cluster_size}
Let $h>0$, and 
let $A\subset \mathbb{R}^d$, such that for all $p\in A$, we have $\diam(A) \leq h\cdot \gamma(p)$.
Then, $|X\cap A| = k \cdot  h^{d+O(1)}$.
\end{lemma}

\begin{proof}
For any $p\in \mathbb{R}^d$, we have that $\gamma(p)$ is the distance between $p$ and $k$-th nearest neighbor of $p$ in $X$.
It follows that the interior of $\ball(p, \gamma(p))$ contains at most $k$ points in $X$ (it contains at most $k-1$ points in $X$ if $p\notin X$).
In particular, the (closed) ball $\ball(p, \gamma(p)/2)$ contains at most $k$  points in $X$.
Let 
\[
r^* = \inf_{p\in A} \gamma(p).
\]
It follows that for all $p\in A$, 
\begin{align}
|X\cap \ball(p, r^*/2)| &\leq k. \label{eq:ball_intersection}
\end{align}

We have by the assumption that $\diam(A)\leq h\cdot r^*$, and thus
$A\subseteq \ball(p^*, R^*)$, for some $p^*\in A$, and some $R^*=2h\cdot r^*$.
For any $0<\alpha<\beta$, we have that any ball of radius $\beta$ in $\mathbb{R}^d$ can be covered by at most $O(\beta/\alpha)^d=(\beta/\alpha)^{d+O(1)}$ balls of radius $\alpha$.
Therefore, $A$ can be covered by a set of at most 
$(R^*/r^*)^{d+O(1)} = h^{d+O(1)}$
balls of radius $r^*/2$.
Combining with \eqref{eq:ball_intersection} it follows that
\[
|X\cap A| = k \cdot h^{d+O(1)},
\]
which concludes the proof.
\end{proof}

\subsection{The main poisoning algorithm}
\label{alg:knn_poison}
We are now ready to describe the main poisoning algorithm against $k$-NN.
For any finite $Y\subset \mathbb{R}^d$, and for any integer $i\geq 0$, let $\OPT_{i}(X)$ be the maximum corruption that can be achieved for $X$ with a poison set of size at most $i$. Let $corruption(X,Y)$ be the corruption of poisoning $X$ by flipping labels of point set $Y\subset X$. Now we describe our poisoning algorithm.

\textbf{Algorithm \AlgKNNFlip~for $k$-NN Poisoning:}
The input consists of $X\subset \mathbb{R}^d$, with $|X|=n$, and $\mylabel : X \to \{1,2\}$.
\begin{description}
    \item{\textbf{Step 1.}}
    Sample the random partition $P$ according to the algorithm in Lemma \ref{lem:multiscale-euclidean}.
    
    \item{\textbf{Step 2.}}
    For any cluster $C\subset X$ in $P$,
    by Lemma \ref{lem:cluster_size} 
    we have that $|C| = k\cdot (\sqrt{d} 2^{8k/\eps})^{d+O(1)} = k\cdot 2^{(d+O(1))8k/\eps}$.
    For any $i\in \{1,\ldots,m\}$, we compute an optimal poisoning, $S_{C,i}\subseteq C$, for $C$ with $i$ poison points via brute-force enumeration.
    Each solution can be uniquely determined by selecting the $i$ points for which we flip their label.
    Thus, the number of possible solutions is at most $2^{|C|} = 2^{k\cdot 2^{(d+O(1))8k/\eps}}$.
    The enumeration can thus be done in time $2^{k\cdot 2^{(d+O(1))8k/\eps}}$, for each cluster in $P$.
    Since there are at most $n$ clusters, the total time is $n\cdot 2^{k\cdot 2^{(d+O(1))8k/\eps}} = n\cdot 2^{2^{O(d+k/\eps)}}$.
    
    \item{\textbf{Step 3.}}
    We next combine the partial solutions computed in the previous step to obtain a solution for the whole pointset.
    This is done via dynamic programming, as follows.
    We order the clusters in $P$ arbitrarily, as $P=\{C_1,\ldots,C_{|P|}\}$.
    For any $i\in \{0,\ldots,|P|\}$, $j\in \{1,\ldots,m\}$, let 
    \[
    A_{i,j} = \OPT_{j}(C_1\cup \ldots \cup C_i).
    \]
    We can compute $A_{i,j}$ via dynamic programming using the formula

   \begin{equation*}
   \scriptsize
    \begin{split}
    A_{i,j} = \left\{\begin{array}{ll}
    \max\limits_{t\in [j]} \left( A_{i-1,t} + \corruption(C_i, S_{C_i,j-t}) \right) & \text{ if } i>0 \\
    0 & \text{ otherwise } 
    \end{array}\right.
    \end{split}
    \end{equation*}
    
    The size of the dynamic programming table is $O(|P|\cdot m) = O(nm)$.
    The same recursion can also be used to compute an optimal $k$-poison, $Y$, for $C_1\cup \ldots \cup C_{|P|}$.
    The algorithm terminates and outputs $Y$ as the final poison for $X$.
\end{description}   
\begin{proof}[Correctness of Dynamic Programming]
By definition, $A_{i,j}$ is the optimal (maximum) corruption with $j$ poison points on the first $i$ clusters $(C_1\cup \ldots \cup C_i)$. $A_{i,0}=0$ for all $i$. Suppose the solution $A_{i-1,j}$ is correct, then $A_{i,j}$ is the maximum of optimal corruption for poisoning the first $i-1$ clusters with $t$ points, plus the corruption using the remaining $j-t$ points on $i$-th cluster. 
\end{proof}

\begin{lemma}
\label{lem:corruption_kNN}
$E[\corruption(X,Y)] \geq \OPT_m(X) - \eps n$.
\end{lemma}

\begin{proof}
Let $Z\subseteq X$ be an optimal $k$-poison for $X$.
Recall that $P$ is the random partition sampled in Step 1.

For any $x\in X$, $i\in \mathbb{N}$, let ${\cal E}_{x,i}$ be the event that the cluster of $P$ containing $x$, does not contain the $i$-nearest neighbors of $x$ in $X$; i.e.~$\NN_i(x,X)\notin P(x)$.
Let also ${\cal E}_x$ be the event that the cluster of $P$ containing $x$, does not contain all of the $k$-nearest neighbors of $x$ in $X$; i.e.
\[
{\cal E}_x={\cal E}_{x,1} \vee \ldots \vee {\cal E}_{x,k}.
\]
Thus
\begin{align}
\Pr[{\cal E}_x] &= \Pr[{\cal E}_{x,1} \vee \ldots \vee {\cal E}_{x,k}] \notag \\
 &\leq \sum_{i=1}^k \Pr[{\cal E}_{x,i}] & \text{(union bound)} \notag \\
 &= \sum_{i=1}^k \Pr[P(x)\neq P(\NN_i(x)] \notag \\
 &\leq \sum_{i=1}^k \frac{\eps \|x-\NN_i(x)\|_2}{k \gamma(p)} & \text{(Lemma \ref{lem:multiscale-euclidean})} \notag \\
 &\leq k \frac{\eps \|x-\NN_k(x)\|_2}{k \gamma(p)}  \notag \\
 &= \eps \label{eq:PrEx}
\end{align}
Let 
\[
X' = \{x\in X:\{\NN_1(x),\ldots,\NN_k(x)\}\not\subseteq P(x)\}.
\]
By \eqref{eq:PrEx} and the linearity of expectation, it follows that 
\begin{align}
E[|X'|] = \sum_{x\in |X|} Pr[{\cal E}_x] \leq \eps|X| = \eps n. \label{eq:EXprime}
\end{align}

Let $Y$ be the poison that the algorithm returns.
Note that $X\setminus X' = C_1\cup\ldots\cup C_{|P|}$.
For any $x\in X\setminus X'$, if $Y$ corrupts $x$ in $X\setminus X'$, then it must also corrupt $x$ in $X$ (since, by the definition of $X$, all $k$-nearest neighbors of $x$ are in $X\setminus X'$).
Thus, $\OPT_m(X)\geq \OPT_m(X\setminus X') \geq \OPT_m(X) - |X'|$, and moreover,
\begin{align*}
\corruption(X,Y) \geq &~ \corruption(X\setminus X', Y) \\
 = &~ \OPT_m(X\setminus X') \\&~ \text{(by the dynamic program)}\\
 \geq &~  \OPT_m(X) - |X'|.
\end{align*}
Combining with \eqref{eq:EXprime} and the linearity of expectation we get
$E[\corruption(X,Y)]\geq \OPT_m(X)-E[|X'|] \geq \OPT_m(X) -\eps n$, which concludes the proof.
\end{proof}

\begin{proof}[Proof of Theorem \ref{thm:kNN}]
The bound on the corruption follows by Lemma \ref{lem:corruption_kNN}.
The running time is dominated by Step 2, which takes time $n\cdot 2^{2^{O(d+k/\eps)}}$.
\end{proof}

\subsection{Poisoning $k$-NN: with train and test datasets}
In this section, we extend our poisoning algorithm to the case with train and test datasets. We provide proof of Theorem \ref{thm:kNN2} by extending the lemmas from the paper. 

\addtocounter{theorem}{-1}
\begin{theorem}%
On input $X_{train}, X_{test}\subset \mathbb{R}^d$, with $|X_{train}|=n_{train}$,
$|X_{test}|=n_{test}$,
and $m\in \mathbb{N}$,
Algorithm \AlgKNNFlip'~computes a $m$-poison against $k$-NN,
with expected corruption $\OPT_m(X_{train}, X_{test})-\eps n$, in time $(n_{train}+n_{test})\cdot 2^{2^{O(d+k/\eps)}}$,
where $\OPT_m(X_{train}, X_{test})$ denotes the maximum corruption 
incurred on $X_{test}$
when all neighbors are chosen from $X_{train}$,
of any $m$-poison on $X_{train}$.
\end{theorem}
\addtocounter{theorem}{1}

Lemmas \ref{lem:mod_Rd} and \ref{lem:Lip_multiscale} hold without any modifications.

We now define $\gamma_i(p)$ as follows:

For any $p \in \mathbb{R}^d$, for any $i \in [k]$, let $\Gamma_i(p)$ be the $i$-th nearest neighbor of $p$ in $X_{train}$, breaking ties arbitrarily, and let

\[
\gamma_i(p) = \|p-\Gamma_i(p)\|_2.
\]

We denote $\gamma_k(p)$ as $\gamma(p)$.

\begin{lemma}[Euclidean multi-scale random partition]
\label{lem:multiscale-euclidean-traintest}
Let $\eps>0$, there exists a random partition $P$ of $X_{train}$, satisfying the following conditions:
\begin{description}
    \item{(1)}
    The following statement holds with probability 1:
    For any $p\in X_{train}$, 
    \[
    \diam(P(p)) \leq \gamma(p)   2^{8k/\eps} O(\sqrt{d})
    \]
    
    \item{(2)}
    For any $p,q\in \mathbb{R}^d$,
    \[
    \Pr[P(p) \neq P(q)] \leq  \frac{\eps \|p-q\|_2}{k \gamma(p)}.
    \]
\end{description}
\end{lemma}

\begin{lemma}
\label{lem:cluster_size-traintest}
Let $h>0$, and 
let $A\subset \mathbb{R}^d$, such that for all $p\in A$, we have $\diam(A) \leq h\cdot \gamma(p)$.
Then, $|X_{train}\cap A| = k \cdot  h^{d+O(1)}$.
\end{lemma}

\begin{proof}
For any $p\in \mathbb{R}^d$, we have that $\gamma(p)$ is the distance between $p$ and $k$-th nearest neighbor of $p$ in $X_{train}$.
It follows that the interior of $\ball(p, \gamma(p))$ contains at most $k$ points in $X_{train}$ (it contains at most $k-1$ points in $X_{train}$ if $p\notin X_{train}$).
In particular, the (closed) ball $\ball(p, \gamma(p)/2)$ contains at most $k$  points in $X_{train}$.
Let 
\[
r^* = \inf_{p\in A} \gamma(p).
\]
It follows that for all $p\in A$, 
\begin{align}
|X_{train}\cap \ball(p, r^*/2)| &\leq k. \label{eq:ball_intersection-traintest}
\end{align}

We have by the assumption that $\diam(A)\leq h\cdot r^*$, and thus
$A\subseteq \ball(p^*, R^*)$, for some $p^*\in A$, and some $R^*=2h\cdot r^*$.
For any $0<\alpha<\beta$, we have that any ball of radius $\beta$ in $\mathbb{R}^d$ can be covered by at most $O(\beta/\alpha)^d=(\beta/\alpha)^{d+O(1)}$ balls of radius $\alpha$.
Therefore, $A$ can be covered by a set of at most 
$(R^*/r^*)^{d+O(1)} = h^{d+O(1)}$
balls of radius $r^*/2$.
Combining with \eqref{eq:ball_intersection-traintest} it follows that
\[
|X_{train}\cap A| = k \cdot h^{d+O(1)},
\]
which concludes the proof.
\end{proof}

\begin{proof}
Let $M=(X,\rho)$ be the metric space obtained by setting $\rho$ to be the Euclidean metric.
By Lemma \ref{lem:mod_Rd} we have
$\beta_M=O(\sqrt{d})$.
Let $P$ be the random partition of $X_{train}$ obtained by applying Lemma 
\ref{lem:Lip_multiscale},
setting $\gamma:X\to\mathbb{R}_{\geq 0}$ where 
$r = B \gamma$,
with 
$B = 2k \beta_M / \eps$,
and
$C = 2^{4k/\eps}$.
By Lemma \ref{lem:1-Lip} we have that $\|\gamma\|_{\Lip}=1$, and thus
$\|r\|_{\Lip}=\|B\gamma\|_{\Lip} = B\|\gamma\|_{\Lip}=B$.
The assertion now follows by straightforward substitution on Lemma \ref{lem:Lip_multiscale}.
\end{proof}

\subsection{The poisoning algorithm with train and test datasets}
\label{alg:knn_poison_traintest}

In this section, we describe the poisoning algorithm that will be used to obtain an m-poison with the guarantees of Theorem \ref{thm:kNN2}. This follows the algorithm \ref{alg:knn_poison} with three major differences:

\begin{itemize}
    \item The $\gamma_i(p)$ function is only defined with respect to the points within $X_{train}$
    \item The random partition in Step 1 is only on $X_{train}$
    \item The corruption in Step 2 is measured only on $X_{test}$ for the test points that fall within the same cluster
\end{itemize}

 For any finite $Y\subset \mathbb{R}^d$, and any integer $i \geq 0$, let $\OPT_{i}(X_{train}, X_{test})$ be the maximum corruption that can be achieved for $X_{train}, X_{test}$ with a poison set size of at most $i$. Let $corruption(X_{train}, X_{test}, Y)$ be the corruption of $X_{test}$ by flipping the labels of $Y \subseteq X_{train}$.

\textbf{Algorithm \AlgKNNFlip~for $k$-NN Poisoning with Train-Test:}
The input consists of $X_{train}$ and $X_{test}$ $\subset \mathbb{R}^d$, with $|X_{train}|=n_{train}$, $|X_{test}|=n_{test}$ and a map $\mylabel : X \to \{1,2\}$ that maps $X_{train}$ and $X_{test}$ to their corresponding labels.
\begin{description}
    \item{\textbf{Step 1.}}
    Sample the random partition of $X_{train}$ - $P$ according to the algorithm in Lemma \ref{lem:multiscale-euclidean-traintest}.
    
    \item{\textbf{Step 2.}}
    For any cluster $C\subset X_{train}$ in $P$,
    by Lemma \ref{lem:cluster_size-traintest} 
    we have that $|C| = k\cdot (\sqrt{d} 2^{8k/\eps})^{d+O(1)} = k\cdot 2^{(d+O(1))8k/\eps}$.
    For any $i\in \{1,\ldots,m\}$, we compute an optimal poisoning, $S_{C,i}\subseteq C$, for $C$ with $i$ poison points via brute-force enumeration.
    Each solution can be uniquely determined by selecting the $i$ points for which we flip their label.
    Thus, the number of possible solutions is at most $2^{|C|} = 2^{k\cdot 2^{(d+O(1))8k/\eps}}$.
    The enumeration can thus be done in time $2^{k\cdot 2^{(d+O(1))8k/\eps}}$, for each cluster in $P_X$.
    For each possible solution, we also measure the corruption of the poisoning on the points in test set that fall within the same cluster, which takes $O(n_{test})$ time.
    Since there are at most $n_{train}$ clusters, the total time is $(n_{train} + n_{tst})\cdot 2^{k\cdot 2^{(d+O(1))8k/\eps}} = (n_{train} + n_{test})\cdot 2^{2^{O(d+k/\eps)}}$.
    
    \item{\textbf{Step 3.}}
    We next combine the partial solutions computed in the previous step to obtain a solution for the whole pointset.
    This is done via dynamic programming, as follows.
    We order the clusters in $P$ arbitrarily, as $P=\{C_1,\ldots,C_{|P|}\}$.
    For any $i\in \{0,\ldots,|P|\}$, $j\in \{1,\ldots,m\}$, let 
    \[
    A_{i,j} = \OPT_{j}(C_1\cup \ldots \cup C_i).
    \]
    We can compute $A_{i,j}$ via dynamic programming using the formula
   \begin{equation*}
   \scriptsize
    \begin{split}
    A_{i,j} = \left\{\begin{array}{ll}
    \max\limits_{t\in [j]} \left( A_{i-1,t} + \corruption(C_i, S_{C_i,j-t}) \right) & \text{ if } i>0 \\
    0 & \text{ otherwise } 
    \end{array}\right.
    \end{split}
    \end{equation*}
    
    The size of the dynamic programming table is $O(|P|\cdot m) = O(nm)$.
    The same recursion can also be used to compute an optimal $k$-poison, $Y$, for $C_1\cup \ldots \cup C_{|P|}$.
    The algorithm terminates and outputs $Y$ as the final poison for $X$.
\end{description}

\begin{lemma}
\label{lem:corruption_kNN_train_test}
$E[\corruption(X_{train},X_{test},Y)] \geq \OPT_m(X_{train}, X_{test}) - \eps n_{test}$.
\end{lemma}

\begin{proof}
Let $Z\subseteq X_{train}$ be an optimal $k$-poison for $X_{test}$.
Recall that $P$ is the random partition sampled at Step 1.

For any $x\in X_{test}$, $i\in \mathbb{N}$, let ${\cal E}_{x,i}$ be the event that the cluster of $P$ containing $x$, does not contain the $i_{th}$-nearest neighbor of $x$ in $X_{train}$; i.e.~$\NN_i(x)\notin P(x)$.
Let also ${\cal E}_x$ be the event that the cluster of $P$ containing $x$, does not contain all of the $k$-nearest neighbors of $x$ in $X_{train}$; i.e.
\[
{\cal E}_x={\cal E}_{x,1} \vee \ldots \vee {\cal E}_{x,k}.
\]
Thus
\begin{align}
\Pr[{\cal E}_x] &= \Pr[{\cal E}_{x,1} \vee \ldots \vee {\cal E}_{x,k}] \notag \\
 &\leq \sum_{i=1}^k \Pr[{\cal E}_{x,i}] & \text{(union bound)} \notag \\
 &= \sum_{i=1}^k \Pr[P(x)\neq P(\NN_i(x)] \notag \\
 &\leq \sum_{i=1}^k \frac{\eps \|x-\NN_i(x)\|_2}{k \gamma(p)} & \text{(Lemma \ref{lem:multiscale-euclidean-traintest})} \notag \\
 &\leq k \frac{\eps \|x-\NN_k(x)\|_2}{k \gamma(p)}  \notag \\
 &= \eps \label{eq:PrEx_traintest}
\end{align}
Let 
\[
X' = \{x\in X_{test}:\{\NN_1(x),\ldots,\NN_k(x)\}\not\subseteq P(x)\}.
\]
By \eqref{eq:PrEx_traintest} and the linearity of expectation, it follows that 
\begin{align}
E[|X'|] = \sum_{x\in |X_{test}|} Pr[{\cal E}_x] \leq \eps|X_{test}| = \eps n_{test}. \label{eq:EXprime_traintest}
\end{align}

From \eqref{eq:EXprime_traintest}, it follows that,
\begin{align*}
\corruption(X_{train}, X_{test}, Y) \geq &~ \corruption(X_{train}, X_{test}\setminus X', Y) \\
 = &~ \OPT_m(X_{train}, X_{test}\setminus X') \\&~ \text{(by the dynamic program)}\\
 \geq &~  \OPT_m(X_{train}, X_{test}) - |X'|.
\end{align*}

Combining with \eqref{eq:EXprime_traintest} and by linearity of expectation we get
$E[\corruption(X,Y)]\geq \OPT_m(X)-E[|X'|] \geq \OPT_m(X) -\eps n_{test}$, which concludes the proof.
\end{proof}

\begin{proof}[Proof of Theorem \ref{thm:kNN2}]
The bound on the corruption on the test set follows Lemma \ref{lem:corruption_kNN_train_test}. The running time is dominated by Step 2 of \ref{alg:knn_poison_traintest} which takes time $(n_{train} + n_{test})\cdot 2^{2^{O(d+k/\eps)}}$.

\end{proof}

\section{Conclusion}
\label{sec:conclusions}

We have introduced an approximation algorithm along with provable guarantees for a label flipping poisoning attack against the geometric classification task of $k$-nearest neighbors. 
Our poisoning framework, specifically the application of approximation algorithms using random metric partitions could also be extended to propose similar defense algorithms.

\bibliography{cccg23_kNN_arxiv}
\bibliographystyle{abbrv}

\end{document}